\documentclass[letterpaper, 10 pt, conference]{ieeeconf}  

\IEEEoverridecommandlockouts                              

\usepackage[linesnumbered,ruled,vlined]{algorithm2e}
\usepackage{graphicx}
\usepackage[caption=false,font=footnotesize,lofdepth,lotdepth]{subfig}

\usepackage{graphicx,adjustbox}
\usepackage{cite}

\usepackage{amsmath}
\usepackage{amssymb}
\usepackage{amsfonts}
\usepackage{graphicx}
\usepackage{color}
\usepackage{url}
\usepackage{soul}
\usepackage[bookmarks=true]{hyperref}
\usepackage[usenames,x11names]{xcolor}
\usepackage{xspace} 
\usepackage{caption}
\usepackage{multicol}
\usepackage[normalem]{ulem}

\usepackage[linesnumbered,vlined,ruled]{algorithm2e}
\usepackage{multirow} 
\usepackage{tikz,pgf}
\usetikzlibrary{arrows,automata,shapes,calc,backgrounds,spy,positioning,shadows}

\graphicspath{{figures/}}

\SetKw{Break}{break}

\newtheorem{theorem}{Theorem}[section]
\newtheorem{proposition}[theorem]{Proposition}

\newtheorem{definition}{Definition}[section]

\newtheorem{remark}[theorem]{Remark}

\newtheorem{example}{Example}[section]

\newtheorem{problem}{Problem}[section]

\newtheorem{assumption}{Assumption}

\newcommand{\disha}{\color{black}}

\newcommand{\CA}[1]{\mathcal{#1}}
\newcommand{\BF}[1]{\mathbf{#1}}
\newcommand{\BB}[1]{\mathbb{#1}}

\newcommand{\BS}[1]{\boldsymbol{#1}}

\newcommand{\True}{\top}

\newcommand{\virtual}{\rhd}

\newcommand{\AP}{{AP}}
\newcommand{\Pref}{{R}}

\newcommand{\buchi}{B\"uchi\ }

\newcommand{\PA}{\mathcal{P}}

\newcommand{\FA}{\mathcal{A}}
\newcommand{\TS}{\mathcal{T}}

\newcommand{\ES}{\mathcal{E}}

\newcommand{\card}[1]{\left| {#1} \right|}
\newcommand{\spow}[1]{2^{#1}}
\newcommand{\lrel}[1]{\left| {#1} \right|_{LR}}

\newcommand{\dto}{\rightrightarrows}
\newcommand{\range}[1]{\left[\left[{#1}\right]\right]}
\newcommand{\any}{\bullet}

\graphicspath{ {./images/} }

\newcommand{\ditto}{---{\tt "}---}

\title{\LARGE \bf
Automata-based Optimal Planning with Relaxed Specifications
}

\author{Disha Kamale$^{1}$, Eleni Karyofylli$^{2}$, and Cristian-Ioan Vasile$^{1}$
\thanks{$^{1}$Disha Kamale and Cristian-Ioan Vasile are with the Mechanical Engineering and Mechanics Department,
Lehigh University, Bethlehem, PA 18015
        {\tt\small \{ddk320, cvasile\}@lehigh.edu}}%
\thanks{$^{2}$Eleni Karyofylli is with the Electrical and Computer Engineering Department,
Lehigh University, Bethlehem, PA 18015
{\tt\small elk222@lehigh.edu}}%
}

\renewcommand{\baselinestretch}{0.96}

\begin{document}

\maketitle
\thispagestyle{empty}
\pagestyle{empty}

\begin{abstract}
In this paper, we introduce an automata-based framework for planning with
relaxed specifications.
User relaxation preferences are represented as \textit{weighted finite state edit systems}
that capture permissible operations on the specification, substitution and deletion of tasks,
{\color{black}with complex constraints on ordering and grouping}.
We propose a \textit{three-way product automaton} construction method that
allows us to compute minimal relaxation policies for the robots using {\disha{standard}} shortest path algorithms.
The three-way automaton captures the robot's motion, specification satisfaction,
and available relaxations at the same time.
Additionally, we consider a bi-objective problem that balances temporal relaxation
of deadlines within specifications with changing and deleting tasks.
Finally, we present \textcolor{black}{the runtime performance and a case study that highlights different modalities of our framework.}

%
\end{abstract}

\section{\uppercase{Introduction}}
\label{sec:intro}


Robots are increasingly required to perform complex tasks with rich temporal
and logical structure. In recent years, automata-based approaches have been widely used
for solving robotic path planning problems wherein an automaton
is constructed from mission specifications posed as temporal logic formulae (e.g. LTL, scLTL, STL, TWTL)~\cite{VaBe-RSS-2014, Guo-2015, Tumova:2014, AkVaBe-ICRA-2016 }.
Using shortest path algorithms on the product models between
abstract robot motion models and specification automata,
optimal satisfying trajectories are synthesized.

This traditional approach, albeit useful, does not consider modifying
mission specifications in case satisfaction is infeasible.
This implies that even the sub-parts of the specification
that might be feasible will not be executed.
In many real-world scenarios, it is often preferable that the robot
performs at least some part of the assigned task even if
it cannot be satisfied in its entirety.
Consider the following mission specification for data collection:
\textit{``Collect data from region $G_1$ and then region $G_2$
and then upload it at $U_1$.
Collect data from $G_3$ and upload it at $U_2$. Always avoid obstacles.''}
In case an obstacle makes it impossible to reach $U_2$,
it is still preferred to receive the data from $G_1$ and $G_2$.
Thus, we need to consider relaxed satisfaction semantics
to handle infeasible mission specifications.



In literature, the problem of specification relaxation has been formulated in various ways.
\emph{Minimum violation} is considered in~\cite{Tumova:2016,Vasile-2017-rules,Tumova:2013}
for self-driving cars,
where policies are computed with minimal rules of the road
violation based on priorities.
Their approach is based on the removal of violating symbols
from the input of the specification automata to produce satisfying runs.
A related problem is considered in~\cite{Hauser:2013,Hauser:2014},
where the focus is on removing fewest geometric constraints
for object manipulation.
In~\cite{Lahijanian-2016}, \emph{minimum revision} of tasks for office robots is explored.
Their approach allows changing of tasks based on user-provided substitution costs.
A similar problem is tackled in~\cite{Kim-2015},
but for infinite horizon case with \buchi automata.
Both works modify the input stream of the specification automaton
to induce feasibility.
\emph{Partial satisfaction}~\cite{Lacerda-2015,Lahijanian:2015}
approaches aim to compute policies that minimize distance to satisfaction
given by paths to accepting states in specification automata.
In a different direction, \cite{Vasile-2017-rules,AkVaBe-ICRA-2016,PeVaBe-WAFR-2016}
consider \emph{temporal relaxation} of deadlines to complete missions.
Their approach introduces annotated automata that capture all deadline relaxations
from specifications, to compute policies with minimal delays.
Some of these works combine relaxation of specifications
with maximizing satisfaction probability~\cite{Lacerda-2015,Lahijanian-2016,Guo-2018b}.
All these works use automata-based techniques.
However, all have specialized approaches that can not be readily combined.
Moreover, they operate on a symbol-by-symbol basis rather than words translations
that capture rich relaxation preferences on groups of tasks.

In this paper, we introduce an automata-based framework
that captures {\color{black}{the notion of relaxation from several existing}} approaches 
and generalizes them to operate on groups of tasks (words). We decompose the problem of robot motion planning
into a high-level planning and a low-level control problem.
As in~\cite{ReyesCastro:2013,Tumova:2013, Vasile-2017-rules,Vasile-2017-twtl,Lahijanian-2016},
our focus is on symbolic path planning.
The robot motion is abstracted as a weighted transition system (TS)
with the regions of interest as states.
Mission specification are given as deterministic finite state automata
obtained from finite horizon temporal logic formulae
such as syntactially co-safe LTL (scLTL) and TWTL formulae.
Users also specify relaxation preferences in the form of
{\color{black} regular expressions (RE) which we translate to}
weighted finite state edit systems (WFSE)~\cite{kari2003finite}
that capture differences between pairs of input words.
The WFSEs determine the sets of permissible edit operations  (substitution or skipping),
on single or groups of tasks, {\color{black}along with} their costs when the mission specification is infeasible. 
{\disha{The user-specified relaxation rules enable the framework to be used in complex situations without much computational expense on semantically understanding the environment and deriving the rules.}}
We introduce a three-way product automaton construction method that captures
the motion of the robot, the specification satisfaction, and possible relaxations
at the same time.
We compute minimal relaxation robot trajectories using shortest path algorithms
on the proposed product model.
Additionally, we leverage the framework in~\cite{Vasile-2017-twtl} for temporal relaxation,
and consider bi-objective optimal synthesis problem that balances relaxation of deadlines
with {\disha{task relaxations}}.

{\disha{This work proposes a framework that brings together the core notions of several automata-based methods for planning with relaxation and allows for handling complex specifications and relaxation preferences. }}
The main contributions of the paper are:
(1) {\color{black}
the formulation of the \emph{minimum relaxation problem}
that unifies several problems from the literature, and
generalizes them to relaxation rules with memory;
}
(2) an automata-based formalism to capture
user relaxation preferences via WFSEs;
(3)
{\color{black}
an automata-based planning framework that uses
a novel three-way product automaton construction between
the motion, specification, and relaxation preference models;
}
and 
(4) case studies that demonstrate different instances
of specification relaxation and the runtime performance.
{\color{black}
To the best of our knowledge, this is the first time
relaxation rules are considered that account for complex ordering
and grouping of sub-tasks when revising mission specifications.
}

\section{\uppercase{Preliminaries}}
\label{sec:prelim}

In this section, we introduce notation used throughout the paper,
and briefly review the main concepts from formal languages,
automata theory, and formal verification.
For a detailed exposition of these topics, we refer the reader
to~\cite{Baier2008,Hopcroft2006} and the references therein.

We denote the range of integer numbers as $\range{a, b} = \{a, \ldots, b\}$,
and $\range{a} = \range{0, a}$.

Let $\Sigma$ be a finite set.
We denote the cardinality and the power set of $\Sigma$
by $\card{\Sigma}$ and $\spow{\Sigma}$, respectively.
A \emph{word} over $\Sigma$ is a finite or infinite sequence of
elements from $\Sigma$. In this context, $\Sigma$ is also
called an \emph{alphabet}. The length of a word $w$ is denoted by
$\card{w}$.
Let $k$, $i\leq j$ be non-negative integers.
The $k$-th element of $w$ is denoted by $w_k$, and
the \emph{sub-word} $w_i,\ldots, w_j$ is denoted by $w_{i, j}$.
Let $I=\{i_0,i_1 \ldots\} \subseteq \range{\card{w}}$.
The \emph{sub-sequence} $w_{i_0}, w_{i_1}\ldots$ is
denoted by $w_I$.
A set of words over an alphabet $\Sigma$ is called
a \emph{language} over $\Sigma$.
The language of all finite words over $\Sigma$
is denoted by $\Sigma^*$.


\begin{definition}[Deterministic Finite State Automaton]
\label{def:dfa}
A deterministic finite state automaton (DFA) is a tuple
$\FA = (S_\FA, s_0^\FA, \Sigma, \delta_\FA, F_\FA)$, where:
$S_\FA$ is a finite set of states;
$s^\FA_0 \in S_\FA$ is the initial state;
$\Sigma$ is the input alphabet;
$\delta_\FA \colon S_\FA \times \Sigma \to S_\FA$ is the transition function;
$F_\FA \subseteq S_\FA$ is the set of accepting states.
\end{definition}

A trajectory of the DFA $\BF{s} = s_0 s_1 \ldots s_{n+1}$ is generated by
a finite sequence of symbols $\BS{\sigma} = \sigma_0 \sigma_1 \ldots \sigma_n$
if $s_0 = s^\FA_0$ is the initial state of $\FA$ and
$s_{k+1} = \delta_\FA(s_k, \sigma_k)$
for all $k \geq 0$.
A finite input word $\BS{\sigma}$ over $\Sigma$ is said to be accepted
by a finite state automaton $\FA$ if the trajectory of $\FA$ generated
by $\BS{\sigma}$ ends in a state belonging to the set of accepting states, i.e., $F_\FA$.
The {\em (accepted) language} of a DFA $\FA$ is
the set of accepted input words denoted by $\CA{L}(\FA)$.



\begin{definition}[Transition System]
A weighted transition system (TS) is a tuple
$\TS = (X, x^\TS_0, \delta_\TS, \AP, h, w_\TS)$, where:
$X$ is a finite set of states;
$x^\TS_0 \in X$ is the initial state;
$\delta_\TS \subseteq X \times X$ is a set of transitions;
$\AP$ is a set of properties (atomic propositions);
$h \colon X \to \spow{\AP}$ is a labeling function;
$w_\TS \colon \delta_\TS \to \BB{Z}_{> 0}$ is a weight function.
\end{definition}

A \emph{trajectory} (or run) of the system is an infinite sequence of
states $\BF{x} = x_0 x_1 \ldots$ such that
$(x_k, x_{k+1}) \in \delta_\TS$
for all $k \geq 0$, and $x_0 = x^\TS_0$.
The set of all trajectories of $\TS$ is $Runs(\TS)$.
A state trajectory $\BF{x}$ generates an \emph{output trajectory} $\BF{o} = o_0 o_1 \ldots$,
where $o_k = h(x_k)$ for all $k \geq 0$.
We also denote an output trajectory by $\BF{o}=h(\BF{x})$.
The {\em (generated) language} corresponding to a TS $\TS$ is
the set of all generated output words, which we denote by $\CA{L}(\TS)$.
We define the weight of a trajectory as
$w_\TS(\BF{x}) = \sum_{k=1}^{\card{\BF{x}}} w_\TS(x_{k-1}, x_k)$.



\section{\textcolor{black}{\uppercase{Background on planning with relaxed Specifications}}}
\label{sec:background}


\textcolor{black}{
In this section, we review temporal logic-based planning problems that consider
specification relaxation in case of infeasibility.
In the subsequent sections, we unify and generalize all these problems,
and propose an automata-based framework amenable to off-the-shelf
synthesis methods instead of customized solutions.
For cohesiviness and clarity, we present the core features
of the relaxed TL planning problems,
in some cases, adapted to finite-time.
}

Throughout the {\color{black} paper}, we assume that the motion of
a robot is captured by a finite weighted transition system $\TS$.

{\color{black}
We consider finite-time specifications expressed using temporal logics (TL),
e.g. scLTL~\cite{Latvala03,Duret.13.atva},
TWTL~\cite{Vasile-2017-twtl}, BLTL~\cite{Tkachev13},
and Finite LTL~\cite{DeGiacomo:2013},
and regular expressions (RE)~\cite{Chen-2012,Hopcroft2006}.
}
We do not provide details on their syntax and semantics,
and instead point the reader to relevant references.
{\color{black}
All \textcolor{black}{of these representations }can be translated to DFAs using off-the-shelf tools.
Thus, we consider specifications given as a DFA $\FA$.
}
\subsection{Canonical Problem (CP)}
\label{sec:canonical}


\begin{problem}[Canonical]
\label{pb:canonical}
Find a trajectory for $\TS$
such that the output trajectory is accepted by $\FA$.\\
\emph{Optimality:} Minimize the weight of the trajectory.
\end{problem}
{\color{black}In the canonical problem, no relaxations are permitted.}
\subsection{Minimum Violation Problem (MVP)}
\label{sec:min-violation}

Let $\BF{o}$ be a word over $2^\AP$,
and
$\varpi$ per symbol violation cost.
The \emph{violation cost} of $\BF{o}$ with respect to $\FA$ is
$\min_{I\subseteq \range{\card{w}}} \varpi \card{I}$
s.t. 
$\BF{o}_{\range{\card{w}} \setminus I} \in \CA{L}(\FA)$.
%
The violation cost of a TS trajectory $\BF{x}$
is induced by the output word $\BF{o}=h(\BF{x})$.
\begin{problem}[Minimum violation]
\label{pb:min-violation}
Find a trajectory for $\TS$
such that a sub-sequence of the output trajectory
is accepted by $\FA$.\\
\emph{Optimality:} Minimize the violation cost of the trajectory.
\end{problem}
\subsection{Minimum Revision Problem (MRP)}
\label{sec:min-revision}

Let $\BF{o}$ be a word over $2^\AP$,
and $c\colon 2^\AP \times 2^\AP \to \BB{R}$ the symbol substitution cost.
The \emph{revision cost} of $\BF{o}$ with respect to $\FA$ is
$\min_{\BF{o}' \in \CA{L}(\FA)} \sum_{i=0}^{\card{\BF{o}}} c(\BF{o}_i, \BF{o}'_i)$
s.t.
$\card{\BF{o}'} = \card{\BF{o}}$,
where $\BF{o}'$ is the \emph{revised word}.

The symbol substitution cost function $c$ is defined such that
there is no penalty for no substitution, i.e., $c(\sigma, \sigma) = 0$
for all $\sigma \in 2^\AP$.
In most cases, $c$ is a non-negative symmetric function,
$c(\sigma_1, \sigma_2) = c(\sigma_2, \sigma_1) \geq 0$ for all
$\sigma_1, \sigma_2 \in 2^\AP$.

\begin{problem}[Minimum revision]
\label{pb:min-revision}
Let
$c$ be the symbol substitution cost.
Find a trajectory for $\TS$
such that a revision of the output trajectory is accepted by $\FA$.\\
\emph{Optimality:} Minimize the revision cost of the trajectory.
\end{problem}


\subsection{Hard-Soft Constraints Problem (HSC)}
\label{sec:hard-soft}

\begin{problem}
\label{pb:hard-soft}
Let $\FA_H$ and $\FA_S$ be two specification DFAs.
Find a trajectory for $\TS$
such that the output trajectory is accepted by $\FA_H$,
and, if possible, by $\FA_S$.\\
\emph{Optimality:} Minimize the cost of the trajectory.
\end{problem}

{\color{black}
We adapt the HSC problem from~\cite{Guo-2015} for finite-time specifications,
where we replace \buchi automata with DFAs.
}


\subsection{Partial Satisfaction (PS)}
\label{sec:partial-satisfation}

Let $\BF{o} \in 2^\AP$.
The \emph{continuation cost} of $\BF{o}$ with respect to
$\FA$ is
$\min_{\BF{o}^c \in (2^\AP)^*} \card{\BF{o}^c}$
s.t.
$\BF{o}' = \BF{o} \BF{o}^c \in \CA{L}(\FA)$,
where $\BF{o}'$ is a \emph{continuation} of $\BF{o}$.

\begin{problem}
\label{pb:partial-satisfation}
Find a trajectory for $\TS$
such that a continuation of the output trajectory
is accepted by $\FA$.\\
\emph{Optimality:} Minimize the cost of the continuation.
\end{problem}




{\color{black}
The problem minimizes the amount of work still needed to satisfy the
specification from partial trajectories.
}

\subsection{Temporal Relaxation (TR)}
\label{sec:temporal-relaxation}


In this section, we review Time window temporal logic (TWTL)~\cite{VaBe-RSS-2014},
a rich specification language for robotics applications with explicit time bounds.
{\color{black}
As opposed \textcolor{black}{to} previous relaxation semantics, temporal relaxation for TWTL
is defined based on the formulae structure (i.e., relaxation of deadlines)
rather than symbol operations on satisfying words.
For brevity, we omit most details and refer to~\cite{Vasile-2017-twtl}.
}
 
The syntax of TWTL formulae over a set of atomic propositions $\AP$ is:
\begin{equation*}
\label{eq:logic-def}
\phi :: = H^d s \mid H^d \lnot s \mid \phi_1 \land \phi_2 \mid \phi_1 \lor \phi_2 \mid \lnot \phi_1
\mid \phi_1 \cdot \phi_2 \mid [\phi_1]^{[a, b]},
\end{equation*}
where $s$ is either the ``true'' constant $\True$
or an atomic proposition in $\AP$;
$\land$, $\lor$, and $\lnot$ are the conjunction, disjunction, and negation
Boolean operators, respectively; $\cdot$ is the concatenation operator;
$H^d$ with $d \in \BB{Z}_{\geq 0}$ is the {\em hold} operator;
and $[\ ]^{[a,b]}$
is the {\em within} operator, $a, b \in \BB{Z}_{\geq 0}$ and $a \leq b$.
See~\cite{Vasile-2017-twtl} for the full description of semantics
and examples.


Let $\phi$ be a TWTL formula and $\BS{\tau}\in \BB{Z}^{m}$,
where $m$ is the number of {\em within} operators contained in $\phi$.
The \emph{$\tau$-relaxation} of $\phi$ is a TWTL formula $\phi(\BS{\tau})$,
where each subformula of the form $[\phi_i]^{[a_i, b_i]}$ is replaced by $[\phi_i]^{[a_i, b_i + \tau_i]}$.



Bottleneck and linear truncated temporal relaxation were introduced
in~\cite{Vasile-2017-twtl,PeVaBe-WAFR-2016}, respectively.
For brevity, we consider the linear temporal relaxation (LTR).
The LTR of $\phi(\BS{\tau})$ is $\lrel{\BS{\tau}} =\sum_j \tau_j$,
where $\phi(\BS{\tau})$ is a $\BS{\tau}$-relaxation of $\phi$.
%
%
%



\begin{problem}
\label{pb:temporal-relaxation}
Find a trajectory for $\TS$
such that the output trajectory satisfies the
relaxed formula $\phi(\BS{\tau})$ for some relaxation
$\BS{\tau}$ of the deadlines in $\phi$.\\
\emph{Optimality:} Minimize the linear temporal relaxation.
\end{problem}

\subsection{\color{black}Planning}
\textcolor{black}{
All the aforementioned problems are solved by constructing a standard product automaton between the motion model $\TS$ and the specification DFA $\FA$.
Planning with relaxed semantics is achieved via custom pre-processing procedures of $\FA$,
and custom shortest path algorithms.
In the following, we show that all these problems can be captured via an additional automata-based
model for \emph{user task relaxation}, and solved using standard shortest path methods applied
on a novel 3-way product.
Moreover, MVP and MRP are restricted to relaxations of a single symbol at a time.
Our framework can handle rich relaxation rules that involve changing groups of symbols (words).
}

\section{\uppercase{Problem Formulation}}
\label{sec:pb-form}

In this section, we introduce an optimal planning problem
for finite system abstractions with temporal logic goals.
The specifications are expressed as DFAs which can be obtained from
multiple temporal logics, e.g., scLTL~\cite{Latvala03,Duret.13.atva}, BLTL~\cite{Tkachev13},
Finite LTL~\cite{DeGiacomo:2013}, TWTL~\cite{Vasile-2017-twtl},
and regular expressions~\cite{Chen-2012,Hopcroft2006}
using off-the-shelf tools, e.g., \emph{spot}~\cite{Duret.13.atva},
\emph{scheck}~\cite{Latvala03},
\emph{pytwtl}~\cite{Vasile-2017-twtl}.
%
%
%
We define a cost function based on
user preferences on task removal and substitution in
case satisfying the given specification is infeasible.
Using the \emph{user task preference} we define an optimal
planning problem over the finite motion model,
where the specification language is enlarged
to ensure feasibility with appropriate penalties.

\begin{definition}[User Task Preference]
\label{def:task-preference}
Let $L$ be a language over the alphabet $2^\AP$.
A \emph{user task preference} is a pair $(\Pref, w_\Pref)$,
where $\Pref \subseteq L \times (2^\AP)^*$ is a relation that
captures how words in $L$ can be transformed to words from $(2^\AP)^*$,
and $w_\Pref \colon  \Pref \to \BB{R}$ represents
the cost of the word transformations.
\end{definition}

The relation $\Pref$ can also be understood as a multi-valued
function $\Pref\colon L \dto (2^\AP)^*$.

{\color{black}
\begin{assumption}
\label{asm:user-pref}
The representation of relation $R$ requires bounded memory.
\end{assumption}

Asm.~\ref{asm:user-pref} is a reasonable requirement in practice,
and allows similar expressivity as finite-time TLs and DFAs.
With general relations $\Pref$, we run into decidability issues~\cite{Hopcroft2006}.
}

{\color{black}
Robot motion is captured by TSs whose weights
represent either duration, distance, energy, or control effort.
For simplicity, in this paper, weights are transition durations.
}

\begin{problem}[Minimum Relaxation]
\label{pb:formulation-min-relax}
Given a transition system $\TS$, a specification DFA $\FA$,
and a task relaxation preference $(\Pref, w_\Pref)$,
find a trajectory $\BF{x} \in Runs(\TS)$ that minimizes
the task cost.
Formally, we have
\begin{equation}
\label{eq:task-cost-optimal-pb}
{
\begin{aligned}
  \min_{\BF{x} \in Runs(\TS)} & w_\Pref(\BF{o}, \BF{o}')\\
   \text{s.t.} \ & \BF{o} \in \CA{L}(\FA)\\
   & \BF{o}' = h(\BF{x}) \in \CA{L}(\TS) \cap \Pref(\BF{o})
\end{aligned}
}
\end{equation}
where $\Pref\colon \CA{L}(\FA) \dto (2^\AP)^*$
and $w_\Pref\colon R \to \BB{R}$.
\end{problem}




Task preferences can be used to substitute and delete
tasks which are associated with words.
These generalize edit-space operations on single symbols to words,
and the optimization problem Problem~\ref{pb:formulation-min-relax}
generalizes the Levenstein distance between languages of finite words.

{\color{black}
User task preferences $(\Pref, w_\Pref)$ can be represented in many ways.
We consider {\disha{the user preferences for relaxation provided as}} regular expressions (RE) and regular grammars that can
be readily translated to automata using standard methods~\cite{Hopcroft2006}.
Consider the following example.
\textcolor{black}{\begin{example}
\label{ex:task-relaxation2}
Suppose the task is to visit region P1 for 1 time unit followed by P2 for 2 time units. Should visiting either or both be not possible, the substitution rules are: Substitute the visit to P1 by visiting Q1 for 2 time units with a penalty of d1, and the visit to P2 by visiting S1 for 2 time units followed by S2 for 1 time unit with a penalty of d2. Formally, $\Pref = ((\any/\any, 0)^*(\{Q1\}/\{P1\},\disha{0}) (\{Q1\} / \epsilon,\disha{d1}) (\{S1\}/\{P2\},0)\allowbreak (\{S1\}/ \{P2\},0)$ ${\disha{(\{S2\}/\epsilon,d2)}}(\any/\any, 0)^*$, where $\any$ represents any symbol in $2^\AP$,
$\epsilon$ denotes a {\disha{empty}} symbol,
$\{Q1\}/\{P1\}$ denotes that $\{P1\}$ is substituted by $\{Q1\}$,
and d1, d2 denote the penalties for the corresponding substitutions.
Note that the transformation can be performed multiple times due to
the outer Kleene star operator. Alternatively, the transformation rules are $P1 \mapsto^{d1} Q1Q1$ and $P2P2 \mapsto^{d2} S1S1S2$.
and the possible alternatives are: a) $r_{1}' = Q1 Q1 P2 P2$, b) $r_{2}' = P1S1S1S2$,  and c) $r_{3}' =  Q1Q1S1S1S2$. {\disha {From Fig. \ref{fig:wfse-examples}, it is evident that the existing approaches that allow symbol-symbol translations (Fig. \ref{fig:wfse-symbol}) cannot capture these relaxations as opposed to the WFSE for word-word translations (Fig. \ref{fig:wfse-word}).}}
\end{example}}
}
{\color{black}
Penalties in user task preferences have multiple interpretations;
they can be additive, multiplicative, or percentage rate with respect
to the TS weights depending on the nature of the {\color{black} tasks and preferred relaxations.}
We use weight computation functions $f_w(\cdot)$
that combine TS and WFSE weights to capture these multiple semantics. 
}

{\color{black}
Next, we show that Prob.~\ref{pb:formulation-min-relax}
captures all problems from Sec.~\ref{sec:background}.

\begin{proposition}
\label{thm:univ-prob}
Problems~\ref{pb:canonical},~\ref{pb:min-violation},~\ref{pb:min-revision},~\ref{pb:hard-soft},~\ref{pb:partial-satisfation}, and~\ref{pb:temporal-relaxation} are instances of Prob.~\ref{pb:formulation-min-relax}
\end{proposition}
\begin{proof}
We provide a constructive proof for each case.
In all cases, we consider additive penalties.
\paragraph{CP} $\Pref= \{(\BF{o}, \BF{o}) \mid \BF{o} \in \CA{L}(\FA)\}$,
and $w_\Pref(\BF{o}, \BF{o}) = w_\TS(\BF{x})$,
where $\BF{x}$ is the trajectory generating $\BF{o}$,
i.e.,  $\BF{o} = h(\BF{x})$.
Equivalently, $R = (\any/\any, 0)^*$.
\paragraph{MVP} $\Pref = \{(\BF{o}, \BF{o}_I) \mid \BF{o}_I \in \CA{L}(\FA), I \subseteq \range{\card{\BF{o}}}\}$
and $w_\Pref(\BF{o}, \BF{o}_{\range{\card{\BF{o}}} \setminus I}) = \varpi \card{I}$.
Equivalently, $R = ((\any/\any, 0) \ | \ (\epsilon/\any, \varpi \card{I}))^*$.
\paragraph{MRP} $\Pref=\{ (\BF{o}, \BF{o}') \mid \BF{o} \in \CA{L}(\FA), \BF{o}' \in (2^\AP)^*, \card{\BF{o}'}=\card{\BF{o}}\}$
and $w_\Pref(\BF{o}, \BF{o}')= \sum_{i=0}^{\card{\BF{o}}} c(\BF{o}_i, \BF{o}'_i)$.
Equivalently, $R = ((\sigma/\sigma', c(\sigma, \sigma')) \mid \forall \sigma \in 2^\AP)^*$.
\paragraph{HSC} $\Pref = \{(\BF{o},\BF{o}) \mid  \BF{o} \in \CA{L}(\FA_H)\}$,
$w_\Pref(\BF{o}_1, \BF{o}_2) =
\begin{cases}
w_\TS(\BF{x}) & \text{if } \BF{o}_2 \in \CA{L}(\FA_H) \cap \CA{L}(\FA_S)\\
M + w_\TS(\BF{x}) & \text{if } \BF{o}_2 \in \CA{L}(\FA_H) \setminus \CA{L}(\FA_S)
\end{cases}
$,
$\BF{x}$ is the trajectory of $\TS$ generating $\BF{o}_2$, and $M \gg 1$.
\paragraph{PS} $\Pref = \{(\BF{o}, \BF{o}_{0,k}) \mid
\forall k \in \range{\card{\BF{o}}-1}\}$,
and $w_\Pref(\BF{o}, \BF{o}_{0, k}) = \card{\BF{o}} - k - 1$.
Equivalently, $R = (\any/\any, 0)^* (\epsilon/\any, 0)^*$ where, $d_O = \card{\BF{o}}-1 $ and ${d_{p}} = k$.
\paragraph{TR}
$\Pref = \{(\BF{o}, \BF{o}') \mid \BF{o} \in \CA{L}(\FA), \BF{o}' \in \CA{L}(\FA_\infty)\}$,
and $w_\Pref(\BF{o}, \BF{o}') = \lrel{\BS{\tau}}$,
where $\BS{\tau}$ is the temporal relaxation of $\phi$
associated with $\BF{o}'$,
and $\FA$, $\FA_\infty$ are the DFAs
for $\phi$ and $\phi(\infty) = \bigvee_{\BS{\tau}} \phi(\BS{\tau})$,
see~\cite{Vasile-2017-twtl} for details.
\end{proof}
}

\section{\uppercase{Unified Automata-based Framework}}
\label{sec:framework}

In this section, we introduce a unified automata-based framework
to capture user preference specifications, and to synthesize
minimal relaxation policies.


\subsection{Relaxation Specification}
\label{sec:relax-specification}

We consider two classes of problems
related to task changes and deadline relaxations.

\subsubsection{Task Relaxation}

In this problem class, we allow parts of the specification
to be substituted and/or removed.
%
%
%
%
Preferences can be given in many formats, e.g., regular expressions and grammars,
{\color{black}see Ex.~\ref{ex:task-relaxation2}
We introduce weighted finite state edit systems to represent
user task relaxation preferences {\color{black} with bounded memory (Asm.~\ref{asm:user-pref})},
where weights capture translation penalties.}

\begin{figure}[t]
	\subfloat[]{
	\scalebox{0.57}{
	\begin{tikzpicture}[->,>=stealth',shorten >=1pt,auto,node distance=4cm, semithick,
        initial text={}]
        \tikzstyle{every state}=[text=black]
        
        \node[initial,state,accepting]   (z0)     {$z_0$};
        \node[state]           (z1) [right of=z0] {$z_1$};
        \node[state]           (z2) [below of=z1] {$z_2$};
        \node[state]           (z3) [left  of=z2] {$z_3$};
        
        \path[->] (z0) edge[out=165,in=120,loop] node[above,right] {$(\{P1\}/\{P1\}, 0)$, $(\{P2\}/\{P2\}, 0)$, $(\{\}/\{\}, 0)$} (z0)
                  (z0) edge             node[below] {\disha{$(\{S1\}/\{P2\}, 0)$}} (z1)
                  (z0) edge             node[sloped,above] {\disha{$(\{Q1\}/\{P1\}, 0)$}} (z3)
                  (z1) edge             node[sloped,above] {\disha{$(\{S1\}/\{P2\}, 0)$}} (z2)
                  (z2) edge             node[sloped,above] {\disha{$(\{S2\}/\epsilon, d2)$}} (z0)
                  (z3) edge[bend left]  node[sloped,below] {\disha{$(\{Q1\}/\epsilon, d1)$}} (z0)
                  ;
    \end{tikzpicture}
    \label{fig:wfse-word}
    }
	}
	\subfloat[]{
	\scalebox{0.57}{
	\begin{tikzpicture}[->,>=stealth',shorten >=1pt,auto,node distance=4cm, semithick,
        initial text={}]
        \tikzstyle{every state}=[text=black]
        
        \node[initial,state,accepting]   (z0)     {$z_0$};
        
        \path[->] (z0) edge[loop above] node[above, text width=6.3cm] {
        $(\{\}/\{\}, 0)$, \quad\quad\ $(\{S1\}/\{S2\}, 0)$,
        $(\{P1\}/\{P1\}, 0)$, $(\{S1\}/\{S1\}, 0)$,
        $(\{P1\}/\{Q1\}, d1)$, $(\{P1, P2\}/\{Q1,P2\}, d1)$,
        $(\{Q1\}/\{Q1\}, 0)$, $(\{P1, P2\}/\{P1, S1\}, d2)$,
        $(\{P2\}/\{P2\}, 0)$, $(\{P1, P2\}/\{Q1, S1\}, d1+d2)$,
        $(\{P2\}/\{S1\}, d2)$
        } (z0)
                  ;
    \end{tikzpicture}
    }
    \label{fig:wfse-symbol}
	}

	\caption{\small{\textcolor{black}{\protect \subref{fig:wfse-word} WFSE for word-word translations, \protect\subref{fig:wfse-symbol} WFSE corresponding to symbol-symbol translations}}}
	\label{fig:wfse-examples}
	\vspace{-15pt}
\end{figure}
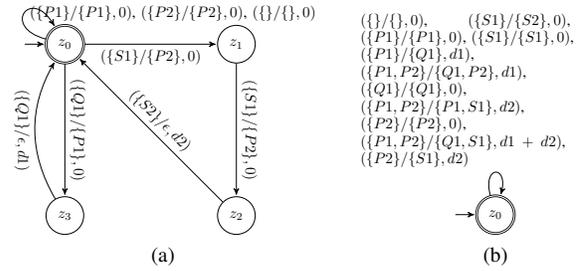 

\begin{definition}[Weighted Finite State Edit System]
\label{def:wfses}
A weighted finite state edit system (WFSE)
is a weighted DFA $\ES = (Z_\ES, z^\ES_0, \Sigma_\ES, \delta_\ES, F_\ES, w_\ES)$, where 
$\Sigma_\ES = \left(2^\AP\cup \{\epsilon\}\right) \times \left(2^\AP\cup \{\epsilon\}\right) \setminus \{(\epsilon, \epsilon)\}$,
$\epsilon$ denotes a missing or deleted symbol,
and $w_\ES\colon \delta_\ES \to \BB{R}$ is the transition weight function.
\end{definition}

The alphabet $\Sigma_\ES$ captures word edit operations (addition, substitution, or deletion of symbols).
A transition $z' = \delta_\ES(z,(\sigma, \sigma'))$
{\color{black}
has input, output symbols $\sigma$ and $\sigma'$.
}
Given a word $\vec{\BS{\sigma}} = (\sigma_0, \sigma'_0) (\sigma_1, \sigma'_1) \ldots (\sigma_r, \sigma'_r) \in \CA{L}(\ES)$, {\color{black}$r=|\vec{\BS{\sigma}}|-1$,}
we call $\BS{\sigma}= \sigma_{i_0} \sigma_{i_1}\ldots \sigma_{i_n} \in (2^\AP)^*$
and $\BS{\sigma}'=\sigma'_{j_0} \sigma'_{j_1}\ldots \sigma'_{j_m} \in (2^\AP)^*$ obtained by removing only the symbol $\epsilon$,
the input and output words,
where
$m, n, r \in \BB{Z}_{>0}$, $m,n < r$,
$0\leq i_0< \ldots<i_m< r$ and $0\leq j_0< \ldots<j_n< r$.
Moreover, we say that $\ES$ transforms $\BS{\sigma}$ into $\BS{\sigma}'$.

Note that WFSE is a special type of finite state transducer 
where the input and output alphabets are the same,
and the empty symbol $\epsilon$ cannot be mapped to itself.
Moreover, the weights capture translation penalties and can
depend on the states and symbol translation pairs.

{\color{black}
We can use standard methods~\cite{Hopcroft2006} to translate REs,
expressing relaxation rules, into WFSEs.}

\subsubsection{ Temporal Relaxation}
Temporal relaxation allows delays with respect to deadlines
in the satisfaction of specifications.
In the following, we consider annotated automata $\FA_\infty$
computed from TWTL formulae~\cite{Vasile-2017-twtl} that
capture all possible deadline relaxations.
{\color{black}Formally, given TWTL formula $\phi$, an annotated DFA $\FA_\infty$ is
a DFA such that $\mathcal{L}(\FA_\infty) = \mathcal{L}(\phi(\infty))$,
where $\phi(\infty)$ is satisfied by a word $\BF{o}$
if and only if $\exists \BS{\tau}' < \infty$ s.t. $\BF{o} \models \phi(\BS{\tau}')$.
When the transition weights of the TS $\TS$ represent (integer) durations,
we construct an extended TS $\widehat{\TS}$ from $\TS$ such that
all transitions have unit weight (duration).
This additional step ensures that transitions of $\FA_\infty$ and $\TS$
are synchronized.
See~\cite{Vasile-2017-twtl} for details.
}

\subsection{Product Automaton Construction}
\label{sec:3-way-pa}

The optimal control policy that takes into account
the user preferences is computed based on a product automaton
between three models:
(a) the motion model (TS) of the robot $\TS$;
(b) the user preferences WFSE $\ES$;
and (c) the specification DFA $\FA$.

\begin{definition}[Three-way Product Automaton]
\label{def:3-way-product}
Given a TS $\TS=(X, x^\TS_0, \delta_\TS, AP, h, w_\TS)$,
a WFSE system $\ES = (Z_\ES, z^\ES_0, \Sigma_\ES, \delta_\ES, F_\ES, w_\ES)$,
and a specification DFA $\FA = (S_\FA, s^\FA_0, 2^\AP, \delta_\FA, F_\FA)$,
the  product automaton is
a tuple $\PA_\ES = (Q^\ES, q^\ES_0, \delta^\ES_{\PA}, F^\ES_\PA, w^\ES_\PA)$
also denoted by $\PA_\ES = \TS \times_\ES A$, where: 
\begin{itemize}
  \item $Q^\ES = X \times Z_\ES \times S_\FA \cup \{q^\ES_0\}$ is the state space;
  \item $q^\ES_0=(\virtual, z_0, s_0)$ is the initial state;
  \item $\delta^\ES_\PA \subseteq Q^\ES \times Q^\ES$ is the set of transitions;
  \item $F^\ES_{\PA} = X \times F_\ES \times F_\FA \subseteq Q^\ES$ is the set of accepting states;
  \item $w^\ES_\PA\colon \delta^\ES_\PA \to \BB{R}$ is the transition weight function.
\end{itemize}
\end{definition}

A transition $((x, z, s), (x', z', s')) \in \delta^\ES_\PA$ if $(x, x') \in \delta_\TS$ or $x' = x_0$,
$z' =\delta_\ES(z, (\sigma, \sigma'))$, $s'= \delta_\FA(s, \sigma')$, and $\sigma = h(x')$.
Note that we introduce a virtual TS state $\virtual$ connected to $x_0$ to avoid the definition of a set of initial states and associated start weights.
{\color{black} State $\virtual$ is only used to simplify notation and implementation, and does not correspond
to an actual state of the robot.}
The weight function is $w^\ES_\PA((q, q')) = f_w(w_\TS(x, x'), w_\ES(z, z'))$,
where $f_w(\cdot)$ is an arbitrary function, $q=(x, z, s)$, $q'=(x', z', s')$, and
$w_\TS(\virtual, x^\TS_0)=1$ by convention.
A trajectory of $\PA_\ES$ is said to be accepting only if it ends in a state that belongs to the set of final states $F^\ES_\PA$.
The projection of the trajectory $\BF{q}=q^\ES_0 q_1 \ldots q_n$ onto the TS $\TS$ is $\BF{x}=x_0 x_1 \ldots x_{n-1}$, where $q^\ES_0$ is the initial state of $\PA_\ES$, and $q_k = (x_{k-1}, z_k, s_k)$, for all $k \in \range{1,n}$.
{\color{black}
Similar to~\cite{VaBe-RSS-2014,Vasile-2017-twtl}, we construct $\PA_\ES$ such that only states that are reachable from the initial state, and reach a final state.
}
\vspace{-8pt}
\begin{algorithm}
\caption{Optimal Planning Algorithm -- $\mathtt{Plan()}$}
\label{alg:control-synthesis}
\DontPrintSemicolon
\KwIn{TS $\TS$, TL specification $\phi$ (scLTL, TWTL, etc.),
user task specification RE $(\Pref, w_\Pref)$,
product weight computation function $f_w$}
\KwOut{optimal policy for the TS $\TS$}
\BlankLine
{\color{black}Translate $\phi$ to DFA $\FA$ using off-the-shelf tools}\\
{\color{black}Translate $(\Pref, w_\Pref)$ to WFSE $\ES$ using standard methods}\\
Compute the three-way PA $\PA_\ES = \TS \times_\ES \FA$ \\
Compute shortest path $\BF{q}^*$ in $\PA_\ES$ from the initial state $q^\ES_0$ to a final in $F^\ES_\PA$ using weights $w^\ES_\PA$\\
Project $\BF{q}^*$ onto $\TS$ to obtain the optimal policy $\BF{x}^*$\;
\Return{$\BF{x}^*$}
\end{algorithm}

\vspace{-25 pt}
\subsection{Optimal Planning}
\label{sec:opt_control}

The general planning procedure $\mathtt{Plan()}$ is outlined in Alg.~\ref{alg:control-synthesis}.
Given a TS $\TS$, {\color{black}TL specification $\phi$, and user task preference $(\Pref, w_\Pref)$}, Alg.~\ref{alg:control-synthesis}
first translates $\phi$ to DFA $\FA$ and $(\Pref, w_\Pref)$ to WFSE $\ES$ (lines 1-2).
Next, it computes the three-way product automaton (line 3).
Similar to the standard procedure, the trajectory is obtained
by projecting onto $\TS$ the shortest path $\BF{q}^*$ from
the initial state $q^\ES_0$ of the PA to an accepting state
in $F^\ES_\PA$ (lines 2-3).
The weights $w^\ES_\PA$ of $\PA_\ES$ used for computing
the shortest path depend on whether we wish to minimize
task or deadline relaxation as shown next.

\subsubsection{Task Cost}
Let $\BF{q}=q_{0} \ldots q_{n}$ be a trajectory of $\PA_\ES$.
The task cost of $\BF{q}$ is
\begin{align}
    C_\ES(\BF{q}) &= \sum_{k=0}^{n-1} c_\ES(q_{k}, q_{k+1}),\\
    c_\ES(q_{k}, q_{k+1}) &= f_w(w_\TS(x_k, x_{k+1}), w_\ES(z_k, z_{k+1})),
\end{align}
where $c_\ES$ is the transition weight, $q_{k} = (x_k, z_k, s_k)$
for all $k \in \range{0, n}$.
The task cost takes into account the penalties associated with substitution and deletion of tasks represented as sub-words of the TS's output words.
The optimal trajectory is computed as $\mathtt{Plan}(\TS, \FA, \ES, c_\ES)$ using Alg.~\ref{alg:control-synthesis}.

\subsubsection{Temporal Relaxation Cost}
In this case, the cost is captured by LTR introduced in Sec.~\ref{sec:temporal-relaxation}
that aggregates all delays captured by the annotated specification DFA $\FA_\infty$.
The PA is denoted by $\PA_{\ES_0}$ and
the optimal trajectory is computed
as $\mathtt{Plan}(\widehat{\TS}, \FA_\infty, \ES_0, c_{TR})$,
where $\widehat{\TS}$ is the extended TS,
and $\ES_0$ is a trivial WFSE with a single node and a pass-through self-loop (leaves symbols unchanged and has weight 0).
The temporal relaxation cost of $\BF{q}$ is
\begin{align}
C_{TR}(\BF{q}) &= |\BF{q}|,\\
c_{TR}(q, q') &= 1, \forall (q, q') \in \delta^{\ES_0}_\PA,
\end{align}
where $\BF{q}$ is a trajectory of $\PA_{\ES_0}$.
Minimizing the length of trajectory $\BF{q}$ is equivalent to
minimizing $\lrel{\BF{\tau}}$.
This follows from the results in~\cite{Vasile-2017-twtl}.


\subsubsection{Bi-objective Cost}
We consider cases where a robot can trade-off between changing tasks and delaying their satisfaction.
The solution combines an annotated specification automaton $\FA_\infty$ with a (non-trivial) relaxation preference WFSE $\ES$ to compute policies in $\widehat{T}$ using 
$\mathtt{Plan}(\widehat{\TS}, \FA_\infty, \ES, c_{bi})$.
The blended cost of a trajectory $\BF{q}$ is
\begin{align}
    \mathcal{C}_{bi}(\BF{q}\mid\lambda) &= \lambda \mathcal{C}_{\ES}(\BF{q}) + (1 - \lambda) \mathcal{C}_{TR}(\BF{q}), \\
    c_{bi}(q, q'\mid\lambda) &= \lambda c_{\ES}(q, q') + (1 - \lambda) c_{TR}(q, q'), 
\end{align}
where the $\lambda \in [0, 1]$ is a parameter that trades-off between the two objectives,
and $c_{bi}$ is the bi-objective transition cost.

We compute the Pareto-optimal trajectories and the Pareto-front using
a parametric Dijkstra's algorithm~\cite{young1991faster}.
The Pareto-front for our bi-objective problem is composed of a finite number
of isolated points in the $c_{TR} \times c_\ES$ cost space.
This follows from the finite size of $\PA_\ES$,
and the fact that any Pareto-optimal policy must be a simple path in $\PA_\ES$.
This implies that $c_{bi}$ as a function
of $\lambda$ is a continuous piecewise-affine function,
where each piece corresponds to a point on the Pareto-front and an interval of $\lambda$ values.

Note that in this setup both types of relaxations are allowed to happen at the same time.

{\color{black}
\begin{remark}
The proposed framework reduces planning to standard shortest path problems on $\PA_\ES$
rather than various custom methods.
Thus, automata-based methods for the canonical case can be immediately used to solve
the \emph{minimum relaxation problem} Pb.~\ref{pb:formulation-min-relax}.
\end{remark}
}
{\disha{
\begin{remark}
As the original MRP problem operates on symbol-symbol basis, the length of the original word and relaxed word needs to be the same (See proof of Prop. 4.1(c)). This is not a requirement for our framework as it operates on groups of symbols (words). Thus, in the following sections, we call the substitution problem as Minimum Word Revision Problem (MWRP). 
\end{remark}}}

\subsection{Complexity Analysis}
\label{sec:complexity-automata-construction}

The construction of the three-way PA $\PA_\ES$, line 3 in Alg.~\ref{alg:control-synthesis}, takes $\mathcal{O}(\card{\delta_{\TS}} \times \card{\delta_{\ES}} \times \card{\delta_{\FA_\infty}})$.
Computing the shortest path $\BF{q}^*$ (line 4) is done with
Dijkstra's algorithm which takes $\mathcal{O}(\card{\delta_{\PA}^\ES} + \card{Q^\ES} \log \card{Q^\ES})$.
Lastly, projection onto $\TS$ (line 5) is linear in the size of $\BF{q}$.

{\color{black}
Crucially, our framework has the same asymptotic complexity as
custom planning methods for the problems in Sec.~\ref{sec:background}.
MVP, MRP, and PS operate one symbol at a time, see the proof of Prop.~\ref{thm:univ-prob}.
Their associated WFSEs have a single state with a self-loop, i.e., $\card{\delta_\ES}=1$ (see Fig.\ref{fig:wfse-symbol}).
Thus, the PA construction complexity degenerates to $\mathcal{O}(\card{\delta_{\TS}} \times \card{\delta_{\FA_\infty}})$.
For TR, the complexity also reduces since the WFSE has a single state;
the deadline relaxation is captured by $\FA_\infty$~\cite{Vasile-2017-twtl}.
Lastly, for HSC, we can choose $\FA_H$ as specification DFA, and the WFSE can have
same structure as $\FA_S$ with penalty $M\gg 1$ if the soft constraint is not satisfied.
For brevity, we omit the formal details.
Thus, the complexity of PA construction becomes
$\mathcal{O}(\card{\delta_{\TS}} \times \card{\delta_{\FA_H}} \times \card{\delta_{\FA_S}})$,
the same as for custom methods (due to the quadratic complexity of language intersection~\cite{Hopcroft2006}).
}


\section{CASE STUDIES}
\label{sec:case-study}
 
In this section, we present \textcolor{black}{a case study highlighting different instances of specification relaxations.}
We consider a self-driving car in an urban environment as shown in Fig.~\ref{fig:with_obstacle}
tasked with visiting specified task regions (green rectangles) while avoiding obstacles.
The possible routes that a vehicle can follow are shown using \textcolor{black}{blue} lines, the permissible directions indicated using arrows,
whereas the waypoints are shown using \textcolor{black}{black circles.  The obstacle $O$ shown as a red cross as well as the local obstacles $O2$ and $O3$ shown using orange crosses}  (Fig.\ref{fig:with_obstacle}) may or may not be present. 
The \textit{`No entry'} symbol and the obstacle (if present) together make it impossible to reach $T1$ and in turn, make the \textcolor{black}{red dotted path infeasible. Similarly, even if the local obstacle $O3$ is not present, the smaller $T1$ region next to $T2$ cannot be stayed at due to the no parking zone. }

The motion of the robot is modeled as a weighted TS $\TS$ (Fig.\ref{fig:TS}) with states and transitions representing the waypoints (white states) and task locations (green states), and roads between them, respectively.
The weights associated with transitions represent their duration.
The transitions to the green task location states have weight one,
and may, e.g., capture parking. Note that for all states in Fig. \ref{fig:TS}, self-transitions exist, but have not been included in the figure for simplicity.
Self-transitions allow the robot to be stationary at all location, {\color{black} except for the purple states 14 and 15}.
 \textcolor{black}{The purple states correspond to local obstacles $O2$ and $O3$ whereas the obstacle $O$ is shown using a red state. } 
The initial state of the robot is state $0$. {\color{black}The transitions shown using yellow arrows and obstacle $O2$ are present only for problems 6-8, node 15 only for problem 8.}

%

In the following, we present multiple scenarios in the self-driving setting that showcase
the CP, MVP, MWRP, HSC, TR, and bi-objective problems.
The specifications, user preferences and the costs for each problem are provided in Tab.~\ref{table:1}.

All specifications are translated to DFAs using off-the-shelf tools~\cite{Duret.13.atva,Vasile-2017-twtl}.
For MWRP, HSC-MWRP, bi-objective problems, the relaxation preferences allow the substitution of $T1$ with $T2$, $T3$, $T4$, and $T5$ with costs 5, 8, 11, and 14, respectively.
For MVP and HSC-MVP, the cost of not visiting $T1$, deletion cost, is 10.
For HSC, the cost of not taking the $bridge$, violation of the soft constraint, is 10.
The preferences and costs are captured by a WFSE with $\card{Z_\ES} = 5$.

We first consider a canonical scenario.
Subsequently, we present extensions within the same environment wherein satisfying the given specification becomes infeasible without relaxation. Feasibility depends on the presence of \textcolor{black} {obstacles $O, O2, O3$ }indicated in Tab.~\ref{table:1}
for each scenario.

For MVP, MWRP, and HSC problems, we consider scenarios with $O$ present
(e.g., road construction, temporary closure),
thereby making visits to $T1$ infeasible.



\begin{figure}[!htb]
	\centering
	\subfloat[]{\includegraphics[scale=0.25,trim=0.5cm 0.6cm 0.6cm 2cm, width=\columnwidth]{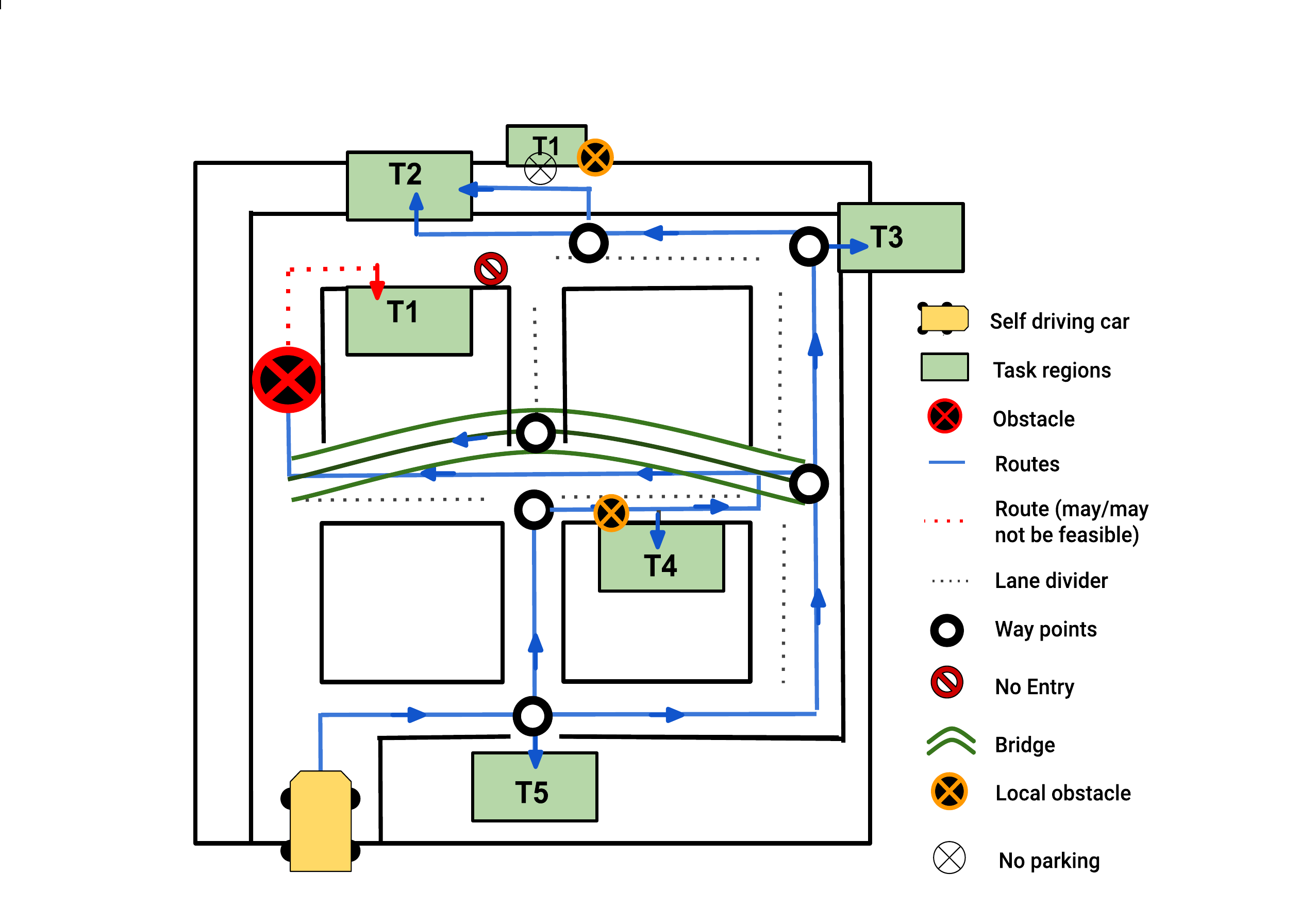}\label{fig:with_obstacle}}\\
	\centering
	\subfloat[]{\includegraphics[scale=0.3,trim=0.5cm 0.5cm 0.5cm 0.5cm, width=\columnwidth]{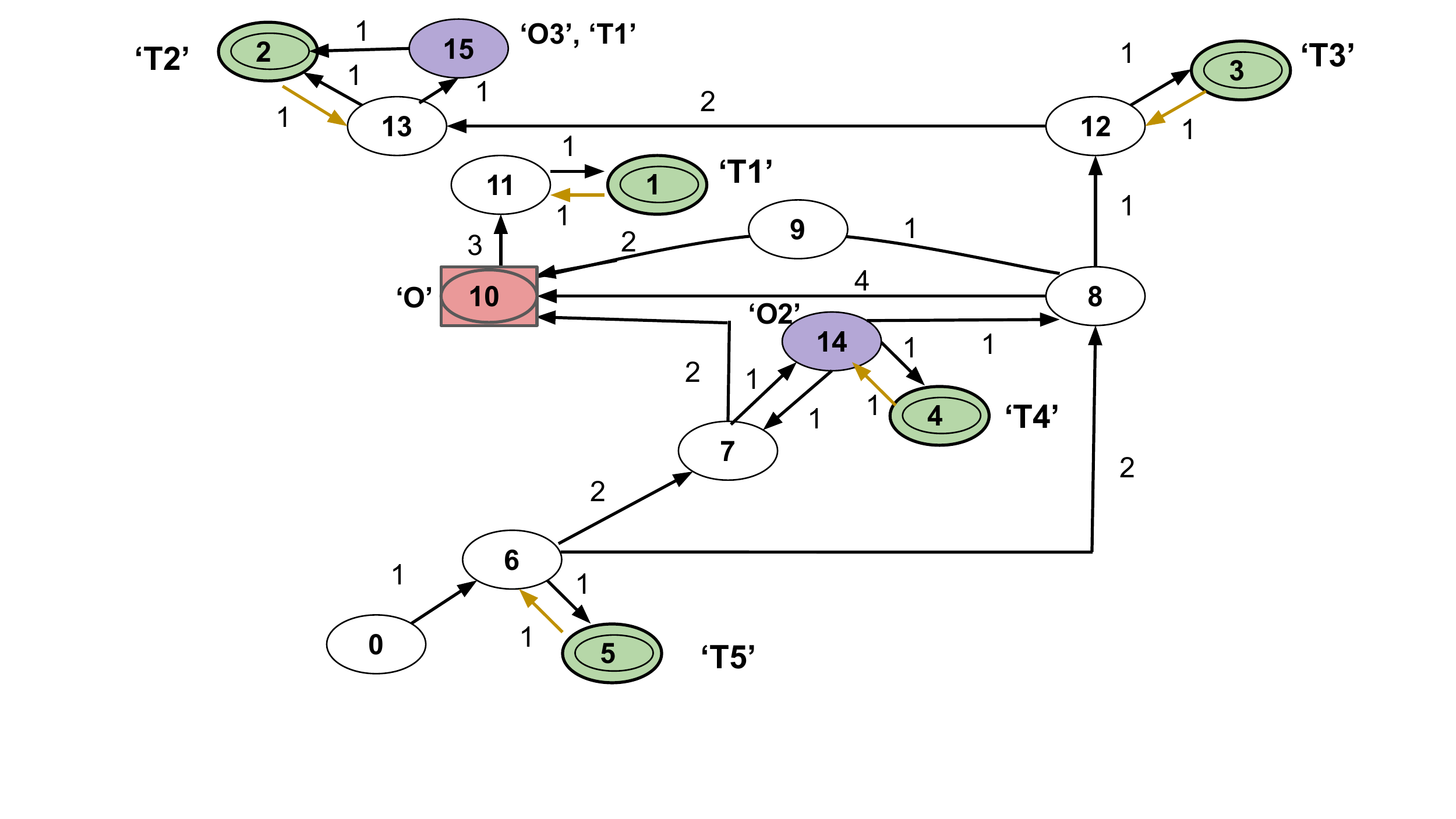}\label{fig:TS}}
	\caption{\protect\subref{fig:with_obstacle}:\small{ Task regions $T1$-$T5$, the obstacles $O,O2,O3$.} 
	Fig.~\protect\subref{fig:TS}: Robot motion model abstraction as a TS, where the green, red, and white nodes correspond to task regions, obstacle and waypoints, respectively. The atomic propositions are shown next to the states. The transition weights indicate the travel duration between states.}
	\label{fig:case_study}
	\vspace{-12pt}
\end{figure}

\subsubsection{Canonical Problem (CP)}
The task specification is ``Visit T1'', i.e., $\phi = \mathbf{F} \; T1$, where $\mathbf{F}$ is the \textit{eventually} operator.
As obstacle $O$ is absent and thus, 
no relaxations are required, this case corresponds to a pass-through operation (no substitutions or deletions) in $\ES$.
The optimal $\TS$ trajectory is $(0, 6, 7, 10, 11, 1)$ which corresponds to the shortest path to $T1$ in Fig. \ref{fig:TS} with the total cost $\mathcal{C}_{\ES}=9$.

\subsubsection{Minimum Violation Problem (MVP)}
We consider the specification ``Visit $T1$ and $T4$ while avoiding obstacles''
that translates to the scLTL formula
$\phi = \neg O \, \mathcal{U} \, T1 \land \neg O \, \mathcal{U} \, T4$.
%
The optimal trajectory of $\TS$ is 
$( 0, 6, 7, 14, 4)$ as $T1$
is not reachable in the presence of $O$.
The optimal cost is 15, including the cost of 10 for skipping $T1$.
\subsubsection{Minimum Word Revision Problem (MWRP)}
In this case, the task specification is ``Visit $T1$ while avoiding obstacles''.
If the task $T1$ is not feasible, revise the task according to the preferences given above.
Here $\phi = \neg O \; \mathcal{U} \; T1$.
The optimal $\TS$ trajectory accepted by $\FA$ is $(0, 6, 8, 12, 13, 2)$
with an optimal cost $\mathcal{C}_{\ES}$ = 11,
where $T1$ is substituted by $T2$ with cost 5.


\subsubsection{Hard-Soft Constraints (HSC)}
This problem is implemented both in the presence and absence of obstacle $O$.
The task specification for both scenarios is
``Visit $T1$ while avoiding obstacles, and, if possible, take the $bridge$".
The specification is $\phi = \phi_h \land \phi_s$,
where $\phi_h = \neg O\, \mathcal{U}\, T1$ and $\phi_s =\mathbf{F}\, bridge$
are the hard and soft constraints, respectively.
The cost of not satisfying $\phi_s$ is 10 and is added to the WFSE.
\\
\textit{HSC-CP}:
Without obstacle $O$, the case is analogous to CP and, thus,
the optimal $\TS$ trajectory is $(0, 6, 8, 9, 10, 11, 1)$
with optimal cost $\mathcal{C}_{\ES} = 10$.\\
\textit{HSC-MWRP}:
In this case, $T1$ is substituted by $T2$ which has the lowest substitution cost, again due to $O$.
Thus, the optimal $\TS$ trajectory is $(0,6,8,12,13,2)$
with optimal cost $\mathcal{C}_{\ES} = 20$
that includes the substitution cost $5$ and
the violation cost $10$ for not going over the $bridge$.\\
\textit{HSC-MVP}:
{\color{black}$\phi_h= (\neg O\, \mathcal{U}\, T1) \land (\neg O\, \mathcal{U}\, T5)$, $\phi_s = \mathbf{F}\, \textit{bridge}$}. With obstacle $O$, only $T5$ can be visited.
In the MVP case, the optimal $\TS$ trajectory is $(0,6,7,5)$
with an optimal cost of $\mathcal{C}_{\ES} = 32$ that includes
the costs of 10 for not visiting $T1$
and of 10 for not taking the $bridge$.

The penalty for not taking the bridge is added to the WFSE
when taking transitions $6 \rightarrow 7$ and $8 \rightarrow 12$,
since these indicate when the robot has diverted from satisfying $\phi_s$.

\subsubsection{Temporal Relaxation (TR)}
In this example, the specification
``Visit and stay in $T2$ for 2 time units within 6 time units."
translates to a TWTL formula $\phi=[{H}^2 T2]^{[0,6]}$,
where ${H}$ is the \textit{hold} operator.
Note that, the minimum travel time to $T2$ from state $0$ is 7 time units, see Fig \ref{fig:TS}.
Thus, the specification is relaxed to $\phi(\tau) = [H^2 T2]^ {[0,6+\tau]}$ with $\tau=3$
obtained by the optimal trajectory $(0, 6, 8, 12, 13, 2,2,2)$.
The optimal cost is $\mathcal{C}_{TR} = 9$ corresponding to $\lrel{\BS{\tau}}=3$.
\textcolor{black}{\subsubsection{Multiple word-word translations}Now consider that the local obstacle $O2$ is present. If the task specification is: ``Visit T4 for 2 consecutive instances and next, visit regions T4 and then T2 and next, visit regions T4 and T1. Avoid obstacles all the time."  The corresponding scLTL specification is: $ \neg O2 \; U\; (\mathbf{F}(T4\land X\; T4) \land  XX \;((\mathbf{F}\; T4 \land \mathbf{F}\; T2 \land (\neg T2\; \mathcal{U} \; T4)) \land X (\mathbf{F} \; T4 \land \mathbf{F} \;T1 \land (\neg O \; \mathcal{U}\; T1) \land (\neg T1\; \mathcal{U} \; T4))))$. If the task is infeasible, the substitution rules are: Substitute the first two instances of T4 (i.e. \textit{$\mathbf{F} T4 \land$ X T4})  by T5 with a total penalty of 6. Substitute the next occurrence of T4 by one T5 and two T3s with a total penalty of 4. Finally, delete the last occurrence of T4 with a penalty of 10 and substitute T1 by T2 with a penalty of 7. Fig. \ref{fig:complex_wfse} shows a wfse constructed from these relaxation rules. Note that all self loops correspond to the translation (\{\},\{\},1). z0-z6 denote the states and the edges represent the allowed edit operations. The trajectory obtained after relaxation is: (0,6,5,5,5,6,8,12,3,3,12,13,13,2,2) with a total cost $\mathcal{C}_\ES= 35$.}


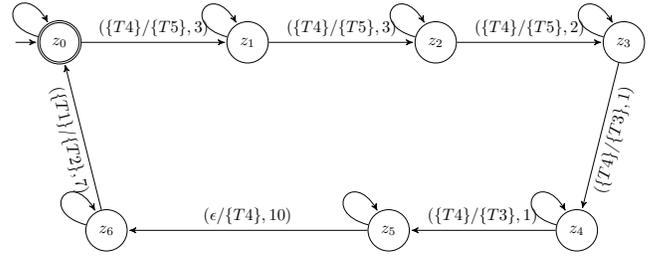
\begin{figure}[t]
    \centering
    \vspace{-15pt}
    \resizebox{0.5\textwidth}{!}{
    	\begin{tikzpicture}[->,>=stealth',shorten >=1pt,auto,node distance=4    	cm, semithick,
        initial text={}]
        \tikzstyle{every state}=[text=black]
        
        \node[initial,state,accepting]   (z0)     {$z_0$};
        \node[state]           (z1) [right of=z0] {$z_1$};
        \node[state]           (z2) [right of=z1] {$z_2$};
        \node[state]           (z3) [right of=z2] {$z_3$};
        \node[state]           (z4) [below of=z3, xshift=-10mm] {$z_4$};
        \node[state]           (z5) [below of=z2, xshift=-10mm] {$z_5$};
        \node[state]           (z6) [below  of=z0, xshift=10mm] {$z_6$};        
        \path[->] (z0) edge[out=165,in=120,loop] (z0)
                  (z0) edge             node[above] {$(\{T4\}/\{T5\}, 3)$} (z1)
                  (z1) edge[out=165,in=120,loop] (z1)
                  (z1) edge             node[above] {$(\{T4\}/\{T5\}, 3)$} (z2)
                  (z2) edge[out=165,in=120,loop] (z2)
                  (z2) edge             node[above] {$(\{T4\}/\{T5\}, 2)$} (z3)
                  (z3) edge[out=165,in=120,loop] (z3)
                  (z3) edge             node[sloped,below] {$(\{T4\}/\{T3\}, 1)$} (z4)
                  (z4) edge[out=165,in=120,loop] (z4)
                  (z4) edge             node[above] {$(\{T4\}/\{T3\}, 1)$} (z5)
                  (z5) edge[out=165,in=120,loop] (z5)
                  (z5) edge             node[above] {$(\epsilon/\{T4\}, 10)$} (z6)
                  (z6) edge[out=165,in=120,loop] (z6)
                  (z6) edge            node[sloped,below] {$(\{T1\}/\{T2\}, 7)$} (z0)
                  ;
            
    \end{tikzpicture} }
    \caption{A WFSE for relaxation rules given in Problem 6}
    \label{fig:complex_wfse}
    
    \vspace{-22pt}
\end{figure}

\textcolor{black}{The above example demonstrates that our framework allows for multiple rules to be taken into account for different instances of the same symbol/word. Also, it highlights how the ordering is considered and retained during relaxation.}

\subsubsection{Bi-objective Problem}
In the absence of the $O$:
``Visit $T1$ for 3 time units between 0 to 5 time units''.
If not feasible, use the substitution preferences.
The DFA $\FA_\infty$ is obtained from the TWTL formula $\phi = [{H}^3 T1]^ {[0,5]}$.
We obtain a set of Pareto-optimal trajectories
and the corresponding intervals for parameter values.
The intervals indicate the range of possible trade-offs
between the two objectives $C_\ES$ and $C_{TR}$
that correspond to the same Pareto-optimal trajectory.
The set of solutions are:
(1) $(0,6,5,5,5,5)$, $\lambda \in [0, 0.25]$, corresponds to
the minimum temporal relaxation $\lrel{\BS{\tau}} = 0$ with $\mathcal{C}_{TR}={5}$
and $\mathcal{C}_\ES=29$;
(2) $(0,6,8,12,3,3,3,3)$, $\lambda \in [0.25, 0.33]$, strikes a balance
between task and temporal relaxations with $\mathcal{C}_{\ES}=20$,
$\mathcal{C}_{TR}={8}$, and $\lrel{\BS{\tau}} = 3$;
(3) $(0,6,8,9,10,11,1,1,1,1)$, $\lambda \in [0.33, 1]$, achieves
minimum task cost $\mathcal{C}_\ES=12$,
$\mathcal{C}_{TR}=12$ and $\lrel{\BS{\tau}} = 7$.
\textcolor{black}{ Having established the core idea, we now consider a word-word translation preference rule. Consider the specification ``Visit T5 for 1s within first 3s from the start and immediately next, proceed to visit T4 for 2s within first 7s followed by T1 for 1s within first 5s of the mission.  The local obstacle O2 should be avoided for the first 4s whereas the global obstacle O should be avoided for all 20s duration of the task."  The corresponding TWTL formula is: ``$ H^{20} \neg O  \land H^4 \neg O2 \land ([H^1 T5]^{[0,3]} \cdot [H^2 T4]^{[0,7]} \cdot [H^1 T1]^{[0,5]})$". The substitution rules are as follows: $T4 \mapsto^3 T5$, $T1 \mapsto^3 T5T3T2$, $T1 \mapsto^4 T3T2$, $T1 \mapsto^5 T3$.}
\textcolor{black}{The trajectories obtained after relaxation are: 1) $\allowbreak(0,6,5,5,5,5,5,\allowbreak5,5,6,12,3,12,13,2)$, \quad 2) $\allowbreak(0,6,5,5,\allowbreak6,7,\allowbreak14,4\allowbreak,4,\allowbreak4,14,8,12,3,3)$, \quad 3) $(0,6,5,5,6,7,14,4,4,4,\allowbreak14,8,12,3,3,13,2)$.}  
\textcolor{black}{\subsubsection{Difference between symbol-symbol and word-word translations} Given that obstacle $O$ is present and $O3$ is absent, the task is to visit T1 for 2 time units. As the route through node 15 is a no parking zone, there are no self-transitions on T1 at node 15. Given same substitution preferences as for MRP and if modelled as a wfse with a single state (see e.g., Fig. \ref{fig:wfse-symbol}) which is analogous to relaxations performed by the existing solutions, the shortest path obtained is (0,6,8,12,13,15,2) which violates the specification as it can pass through T1 (node 15) but not stay there. However,  our framework with a wfse model similar to Fig.\ref{fig:wfse-word} allows for this situation to be taken into account and the resultant trajectory is (0,6,8,12,13,2,2).}

\begin{table}[tb]
\centering
  \resizebox{\columnwidth}{!}{\begin{tabular}{|c|c|c|c|c|} 
  \hline
  User preference & Specification ($\phi$) & $O$? & Optimal trajectory	& cost \\ [0.5ex] 
  \hline\hline
  CP & $\mathbf{F}\, T1$ & No & $\{0, 6, 7, 10, 11, 1\}$ & 9\\
  \hline
  MVP & $(\lnot O \; \mathcal{U} \; T1) \land (\lnot O \; \mathcal{U} \; T4)$ & Yes &  $\{0, 6, 7, 4\}$ & 14\\
  \hline
  MWRP &  $\lnot O \; \mathcal{U} \; T1$ & Yes & $\{0, 6, 8, 12, 13, 2\}$ & 11\\
  \hline
  HSC-CP & $\phi_s = \mathbf{F} bridge$ , $\phi_h = \lnot O \; \mathcal{U} \; T1$ & No & $\{0, 6, 8, 9, 10, 11, 1\}$ & 10\\
  \hline
  HSC- MWRP & \ditto & Yes & $\{0,6,8,12,13,2,\}$ & 20\\
  \hline
  HSC-MVP & $\phi_s = \mathbf{F} bridge$ , $\phi_h = \lnot O \; \mathcal{U} \; T1 \; \land \lnot O \; \mathcal{U} \; T5$ & Yes & $\{0,6,7,5\}$ & 32\\
  \hline  
  TR & $[H^2 T2]^{[0,6]}$ & Yes & $\{0,6,8,12,13,2,2,2\}$ & 9\\
  \hline
  \end{tabular}}
  \caption{\small{User specifications, preferences, and the corresponding costs}}
  \label{table:1}
  \vspace{-18pt}
\end{table}

\begin{figure}[t]
    \vspace{-10pt}
	\centering
	\subfloat[]{\includegraphics[width=0.5\linewidth]{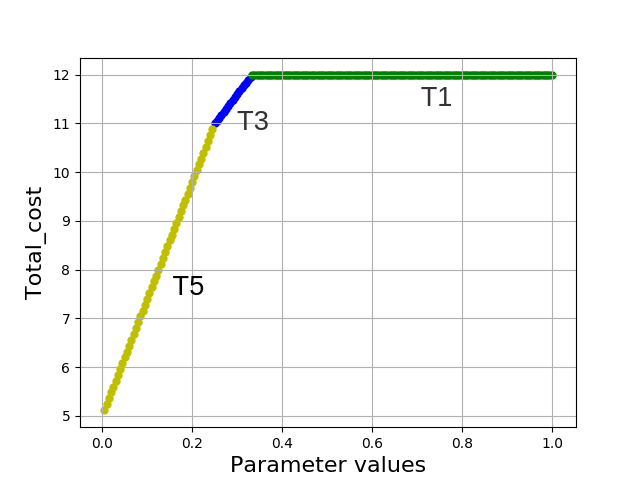}\label{fig:cost_param}}
	\centering
	\subfloat[]{\includegraphics[width=0.5\linewidth]{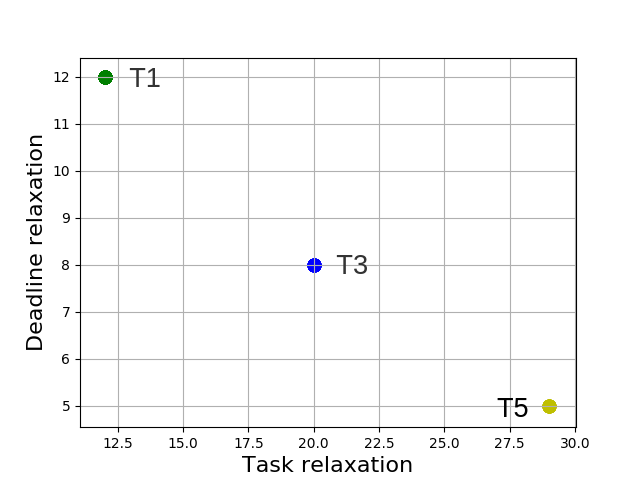}\label{fig:pareto}}
	\caption{\small{\protect\subref{fig:cost_param}: Total cost $\mathcal{C}_{bi}$ as a function of parameter $\lambda$, 
	 \protect\subref{fig:pareto}: Pareto front and the corresponding trajectories. \footnotesize{(Note: The points on the Pareto front correspond to the respective colored line segments in \ref{fig:cost_param})}}}
	\label{fig:Multi-objective Scenario}
	\vspace{-15pt}
\end{figure} 

\section{Runtime Performance Study}
\textcolor{black}{In this section, we study the effect of varying the sizes of $\TS$ and $\ES $ on the size and time taken for $\PA_\ES$ construction. This study was run on Dell Precision 3640 Intel i9-10900K with 64 GB RAM using python 2.7.12. We first keep the WFSE size constant at $\delta_{\ES} = 8$ and vary the size of TS from $\card{X}=118$ to $\card{X}=300$ corresponding to which the $\PA_\ES$ size goes from  $\card{Q^\ES} = 352$ to $\card{Q^\ES} = 898$.  Note that $\PA_\ES$ construction retains only the reachable states. Whereas, the standard Cartesian product between $\TS, \ES, \text{ and }\FA$ varies from 1872 nodes to 4784 nodes. Similarly, when the TS is kept constant at $\card{X}=22$ and the WFSE states are increased from $\delta_{\ES} = 14$ to $\delta_{\ES} = 453$, which corresponds to having n=453 substitution rules taken into account. For this, the $\PA_\ES$ size varies from $\card{Q^\ES} = 106$ to $\card{Q^\ES} = 986$ whereas the Cartesian product size increases from 616 nodes to 19932 nodes agreeing with the complexity analysis results from Section \ref{sec:complexity-automata-construction}. These results are shown in Fig. \ref{fig:model_vs_pa} where the model size refers to the number of states in the TS and the WFSE. Fig. \ref{fig:model_vs_pa_construct} presents the duration for $\PA_\ES$ construction as a function of number of edges in the TS (shown in blue) and the WFSE (shown in green).}

\begin{figure}[t]
	\centering
	\vspace{-20pt}
	\subfloat[]{\includegraphics[width=0.5\linewidth]{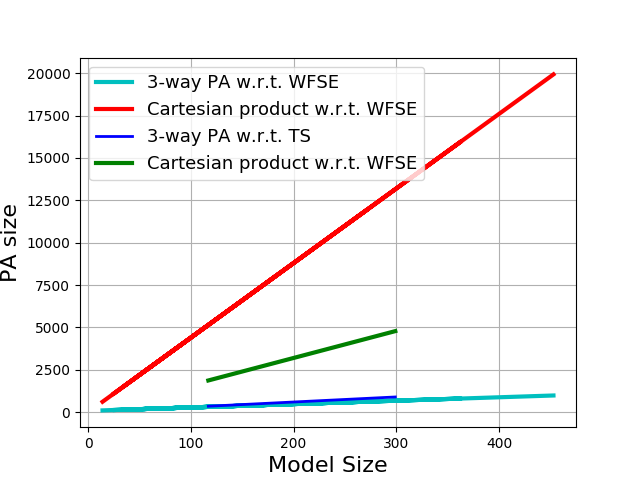}\label{fig:model_vs_pa}}
	\centering
	\subfloat[]{\includegraphics[width=0.5\linewidth]{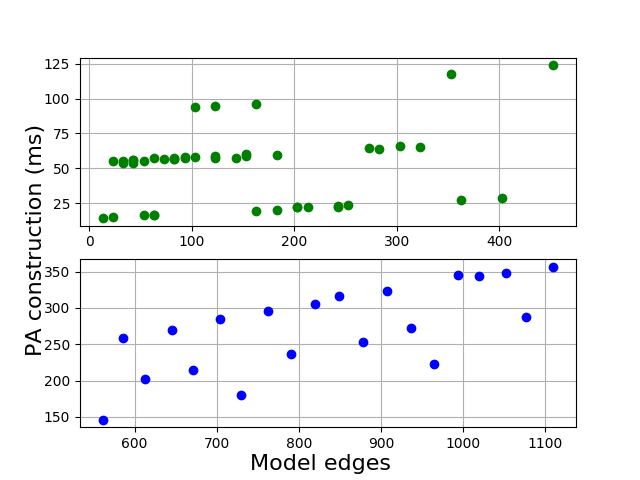}\label{fig:model_vs_pa_construct}}
	\caption{\small{\protect\subref{fig:model_vs_pa}:  $\PA_\ES$ size as a function of $\TS$ and $\ES$ sizes, \protect \subref{fig:model_vs_pa_construct}: $\PA_\ES$ construction duration as a function of $\TS$ (blue) and $\ES$ sizes (green).}}
 	\label{fig:case_study2}
	\vspace{-10pt}
\end{figure}

\section{CONCLUSIONS}
\label{sec:conclusions}
This work studies different existing approaches for specification relaxation and presents a unified and generalised automata-based framework that takes into account different instances of task and temporal relaxations and the trade-off between them.
We propose the construction of a three-way product automaton that allows for word to word translations and temporal relaxations simultaneously.
We demonstrate different instances of specification relaxation through case studies and show that the three-way product automaton construction scales linear in time with respect to transition system size.
Future work includes extending the framework for control synthesis for stochastic systems
while taking into account satisfaction probability.









\bibliographystyle{IEEEtran}
\bibliography{relaxed_specs}

\end{document}